\documentclass[letterpaper]{article}
\usepackage{uai2019}
\usepackage{algorithmic, algorithm}
\usepackage[margin=1in]{geometry}
\usepackage{hyperref}
\usepackage[numbers]{natbib}
\usepackage{tikz}
\usepackage{times}
\usepackage{amsmath, amssymb, amsthm}
\usepackage{float}
\usepackage{stmaryrd}
\usepackage{xspace}
\usepackage{subcaption}
\usepackage{caption}
\usepackage{siunitx}

\DeclareMathOperator{\EE}{\mathbb{E}}
\newtheorem{assumption}{Assumption}
\newtheorem{proposition}{Proposition}

\newtheorem{theorem}{Theorem}

\newtheorem{definition}{Definition}
\newcommand{\supp}[1]{\text{supp}(#1)}
\newcommand{\expect}[2]{\EE_{#1}[#2]}
\newcommand{\mdp}{\mathcal{M}}
\newcommand{\dist}{\mathcal{D}}

\newcommand{\copier}{\textsf{CoPiEr}\xspace}
\newcommand{\update}{\textsf{UPDATE}\xspace}
\newcommand{\exchange}{\textsf{EXCHANGE}\xspace}
\newcommand{\interactive}{\textsf{INTERACTIVE}\xspace}
\title{Co-training for Policy Learning}

\author{
  \textbf{Jialin Song}\textsuperscript{\textdagger} \qquad
  \textbf{Ravi Lanka}\textsuperscript{\textdaggerdbl} \qquad
  \textbf{Yisong Yue}\textsuperscript{\textdagger} \qquad
  \textbf{Masahiro Ono}\textsuperscript{\textdaggerdbl} \qquad
  \\
  \textsuperscript{\textdagger} California Institute of Technology \\
  \textsuperscript{\textdaggerdbl}Jet Propulsion Laboratory, California Institute of Technology
}          

\begin{document}

\maketitle

\begin{abstract}
We study the problem of learning sequential decision-making policies in settings with multiple state-action representations. Such settings naturally arise in many domains, such as planning (e.g., multiple integer programming formulations) and various combinatorial optimization problems (e.g., those with both integer programming and graph-based formulations). Inspired by the classical co-training framework for classification, we study the problem of co-training for policy learning. We present sufficient conditions under which learning from two views can improve upon learning from a single view alone. Motivated by these theoretical insights, we present a meta-algorithm for co-training for sequential decision making. Our  framework is compatible with both reinforcement learning and imitation learning. We validate the effectiveness of our approach across a wide range of tasks, including discrete/continuous control and combinatorial optimization.
\end{abstract}


\section{INTRODUCTION}
\label{sec:intro}
Conventional wisdom in problem-solving suggests that there is more than one way to look at a problem. For sequential decision making problems, such as those in reinforcement learning and imitation learning, one can often utilize multiple different state-action representations to characterize the same problem. 
A canonical application example is learning solvers for hard optimization problems such as combinatorial optimization \citep{he2014learning, mirhoseini2017device, dai2017learning, song2018learning,balunovic2018learning}. It is well-known in the operations research community that many combinatorial optimization problems have multiple formulations. For example, the maximum cut problem admits a quadratic integer program as well as a linear integer program formulation \citep{boros1991max, de2007linear}. Another example is the traveling salesman problem, which admits multiple integer programming formulations \citep{orman2007survey, oncan2009comparative}.  One can also formulate many problems using a graph-based representation (see Figure \ref{fig:mvc}).
Beyond learning combinatorial optimization solvers, other examples with multiple state-action representations include robotic applications with multiple sensing modalities such as third-person view demonstrations \citep{stadie2017third} and multilingual machine translation \citep{johnson2017google}.

In the context of policy learning, one natural question is how different state-action formulations impact learning and, more importantly, how learning can make use of multiple formulations. This is related to the co-training problem \citep{blum1998combining}, where different feature representations of the same problem enable more effective learning than using only a single representation \citep{wan2009co, kumar2011co}. While co-training has received much attention in classification tasks, little effort has been made to apply it to sequential decision making problems. One issue that arises in the sequential case is that some settings have completely separate state-action representations while others can share the action space.

In this paper, we propose \copier (co-training for policy learning), a meta-framework for policy co-training that can incorporate both reinforcement learning and imitation learning as subroutines.   
Our approach is based on a novel theoretical result that  integrates and extends results from PAC analysis for co-training \citep{dasgupta2002pac}  and general policy learning with demonstrations \citep{kang2018policy}. To the best of our knowledge, we are the first to formally extend the co-training framework to policy learning.

Our contributions can be summarized as:
\vspace{-0.1in}
\begin{itemize}
    \item We present a formal theoretical framework for policy co-training. Our results include:
    1) a general theoretical characterization of policy improvement, and 2) a specialized analysis in the shared-action setting to explicitly quantify the performance gap (i.e., regret) versus the optimal policy.
    These theoretical characterizations shed light on rigorous algorithm design for policy learning that can appropriately exploit multiple state-action representations.
    \item We present \copier (co-training for policy learning), a meta-framework for policy co-training. We specialize \copier in two ways:
    1) a general mechanism for policies operating on different representations to provide demonstrations to each other, and 2) a more granular approach to sharing demonstrations in the shared-action setting.
    \item We empirically evaluate on problems in combinatorial optimization and discrete/continuous control.  We validate our theoretical characterizations to identify when co-training can improve on single-view policy learning.  We further showcase the practicality of our approach for the combinatorial optimization setting, by demonstrating superior performance compared to a wide range of strong learning-based benchmarks as well as commercial solvers such as Gurobi.
\end{itemize}

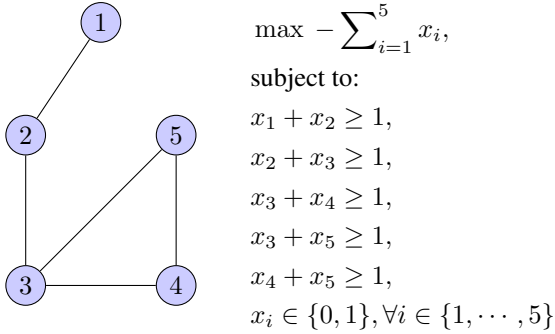
\begin{figure}[t]
    \centering
    \begin{minipage}{0.23\textwidth}
    \centering
    \begin{tikzpicture}[main_node/.style={circle,fill=blue!20,draw,minimum size=1.5em,inner sep=1.5pt}]

    \node[main_node] (3) at (0, 0) {$3$};
    \node[main_node] (2) at (0, 2)  {$2$};
    \node[main_node] (5) at (2, 2) {$5$};
    
    \node[main_node] (4) at (2, 0) {$4$};
    \node[main_node] (1) at (1, 3.5)  {$1$};

    \draw (1) -- (2);
    \draw (2) -- (3);
    \draw (3) -- (4);
    \draw (4) -- (5);
    \draw (3) -- (5);
    
    \end{tikzpicture}
    \end{minipage}
    \begin{minipage}{0.23\textwidth}
    \centering
    \begin{align*}
    &\max\ -\sum\nolimits_{i=1}^5 x_i, \\
    &\text{subject to:} \\
    &x_1 + x_2 \ge 1, \\
    &x_2 + x_3 \ge 1, \\
    &x_3 + x_4 \ge 1, \\
    &x_3 + x_5 \ge 1, \\
    &x_4 + x_5 \ge 1, \\
    &x_i \in \{0, 1\}, \forall i \in \{1, \cdots, 5\}
    \end{align*}
    \end{minipage}
\caption{Two ways to encode minimum vertex cover (MVC) problems. Left: policies learn to operate directly on the graph view to find the minimal cover set \cite{khalil2016learning}. Right: we express MVC as an integer linear program, then polices learn to traverse the resulting combinatorial search space, i.e., learn to branch-and-bound \cite{he2014learning,song2018learning}.}
\label{fig:mvc}
\end{figure}

\section{RELATED WORK}
\label{sec:related}
\paragraph{Co-training}
Our work is inspired by the classical co-training framework for classification \citep{blum1998combining}, which utilizes two different feature representations, or views, to effectively use unlabeled data to improve the classification accuracy. Subsequent extensions of co-training includes co-EM \citep{nigam2000analyzing} and co-regularization \citep{sindhwani2005co}.  Co-training has been widely used in natural language processing  \citep{wan2009co, kiritchenko2011email}, clustering \citep{kumar2011co, liu2013multi}, domain adaptation \citep{chen2011co} and game playing \citep{lample2017playing}. For policy learning, some related ideas have been explored where multiple estimators of the value or critic function are trained together \citep{wang2016dueling,van2016deep}.

In addition to the empirical successes, several previous works also establish theoretical properties of co-training \citep{blum1998combining, balcan2005co, dasgupta2002pac, wang2010new}. Common assumptions in these analyses include: 1) each view is sufficient for learning a good classifier on its own, and 2) conditional independence of the features given the labels. Recently, there are works considering weakened assumptions, such as allowing for weak dependencies between the two views \citep{blum2017efficient}, or relaxing the sufficiency condition \citep{wang2013co}.

\paragraph{Policy Learning for Sequential Decision Making}
Sequential decision making pertains to tasks where the policy performs a series of actions in a stateful environment. A popular framework to characterize the interaction between the agent and the environment is a Markov Decision Process (MDP). There are two main approaches for policy learning in MDPs: reinforcement learning and imitation learning.  For both reinforcement learning and imitation learning, we show that co-training on two views can provide improved exploration in the former and  surrogate demonstrations in the latter, in both cases leading to superior performance.

Reinforcement learning uses the observed environmental rewards to perform policy optimization. Recent works include Q-Learning approaches such as deep Q-networks \citep{mnih2013playing}, as well as policy gradient approaches such as DDPG \citep{lillicrap2015continuous}, TRPO \citep{schulman2015trust} and PPO \citep{schulman2017proximal}. Despite its successful applications to a wide variety of tasks including playing games \citep{mnih2013playing, silver2016mastering}, robotics \citep{levine2016end, kober2013reinforcement} and combinatorial optimization \citep{dai2017learning, mirhoseini2017device}, high sample complexity and unstable learning pose significant challenges in practice \citep{henderson2018deep}, often causing learning to be unreliable. 

Imitation learning uses demonstrations (from an expert) as the primary learning signal.  One popular class of algorithms is reduction-based \citep{daume2009search,ross2010efficient, ross2011reduction, ross2014reinforcement, chang2015learning}, which generates cost-sensitive supervised examples from demonstrations. Other approaches include estimating the expert's cost-to go \citep{sun2017deeply},  inverse reinforcement learning \cite{abbeel2004apprenticeship,ho2016generative,ziebart2008maximum}, and behavioral cloning \cite{syed2008game}. 
One major limitation of imitation learning is the reliance on demonstrations. One solution is to combine imitation and reinforcement learning \citep{le2018hierarchical, kang2018policy, cheng2018fast, nair2018overcoming} to learn from fewer or coarser demonstrations.

\section{BACKGROUND \& PRELIMINARIES}
\label{problem}
\paragraph{Markov Decision Process with Two State Representations.} A Markov decision process (MDP) is defined by a tuple $(\mathcal{S}, \mathcal{A}, \mathcal{P}, r, \gamma, \mathcal{S}_T)$. Let $\mathcal{S}$  denote the state space, $\mathcal{A}$ the action space, $\mathcal{P}(s'|s, a)$  the (probabilistic) state dynamics,  $r(s,a)$  the reward function, $\gamma$  the discount factor and (optinal) $\mathcal{S}_T$ a set of terminal states where the decision process ends. We consider both stochastic and deterministic MDPs. An MDP with two views can be written as $\mathcal{M}^A = (\mathcal{S}^A, \mathcal{A}^A, \mathcal{P}^A, r^A, \gamma^A, \mathcal{S}^A_T)$ and $\mathcal{M}^B = (\mathcal{S}^B, \mathcal{A}^B, \mathcal{P}^B, r^B, \gamma^B, \mathcal{S}^B_T)$. To connect the two views, we make the following assumption about the ability to translate trajectories between the two views.


\begin{assumption}
\label{assumption:mapping}
For a (complete) trajectory in $\mathcal{M}^A$, $\tau^A = (s_0^A, a_0^A, s_1^A, a_1^A, \cdots, s_n^A)$, there is a function $f_{A\rightarrow B}$ that maps $\tau^A$ to its corresponding (complete) trajectory $\tau^B$ in the other view $\mathcal{M}^B$: $f_{A\rightarrow B}(\tau^A) = \tau^B = (s_0^B, a_0^B, s_1^B, a_1^B, \cdots, s_m^B)$. 
The rewards for $\tau^A$ and $\tau^B$ are the same under their respective reward functions, i.e., $\sum\nolimits_{i=0}^{n-1}r^A(s_i^A, a_i^A) = \sum\nolimits_{j=0}^{m-1}r^B(s_j^B, a_j^B)$.  Similarly, there is a function $f_{B\rightarrow A}$ that maps trajectories in $\mathcal{M}^B$ to $\mathcal{M}^A$ which preserves the total rewards. 
\end{assumption}

Note that in Assumption \ref{assumption:mapping}, the length of $\tau^A$ and $\tau^B$ can be different because of different state and action spaces.

\paragraph{Combinatorial Optimization Example.}
Minimum vertex cover (MVC) is a combinatorial optimization problem defined over a graph $G=(V, E)$. A cover set is a subset $U\subset V$ such that every edge $e\in E$ is incident to at least one  $v\in U$. The objective is to find a $U$ with the minimal cardinality. For the graph in Figure \ref{fig:mvc}, a minimal cover set is $\{2, 3, 4\}$.

There are two natural ways to represent an MVC problem as an MDP. The first is graph-based \citep{dai2017learning} with the action space as $V$ and the state space as sequences of vertices in $V$ representing partial solutions. The deterministic transition function is the obvious choice of adding a vertex to the current partial solution. The rewards are -1 for each selected vertex. A terminal state is reached if the selected vertices form a cover.

The second way is to formulate an integer linear program (ILP) that encodes an MVC problem:
\begin{align*}
    &\max\ -\sum\nolimits_{v\in V} x_v, \\
    &\text{subject to}: \\
    &x_u + x_v \ge 1, \forall e = (u, v) \in E,\\
    &x_v \in \{0, 1\}, \forall v \in V.
\end{align*}
We can then use branch-and-bound  \citep{land2010automatic} to solve this ILP, which represents the optimization problem as a search tree, and explores different areas of a search tree through a sequence of branching operations. The MDP states then represent current search tree, and the actions correspond to which node to explore next. The deterministic transition function is the obvious choice of adding a new node into the search tree. The reward is 0 if an action does not lead to a feasible solution and is the objective value of the feasible solution minus the best incumbent objective if an action leads to a node with a better feasible solution. A terminal state is a search tree which contains an optimal solution or reaches a limit on the number of nodes to explore.

The relationship between solutions in the two formulations are clear. For a graph $G=(V, E)$, a feasible solution to the ILP corresponds to a vertex cover by selecting all the vertices $v\in V$ with $x_v = 1$ in the solution. This correspondence ensures the existence of mappings between two representations that satisfy Assumption \ref{assumption:mapping}.

Note that, despite the deterministic dynamics, solving MVC and other combinatorial optimization problems can be extremely challenging due to the very large state space. Indeed, policy learning for combinatorial optimization is a topic of active research \cite{khalil2016learning,he2014learning,song2018learning,mirhoseini2017device,balunovic2018learning}.

\paragraph{Policy Learning.} We consider policy learning over a distribution of MDPs.  For instance, there can be a distribution of MVC problems. Formally, we have a distribution $\mathcal{D}$ of MDPs that we can sample from (i.e., $\mathcal{M}\sim\mathcal{D}$).
For a policy $\pi$, we define the following terms:
\begin{align*}
\eta(\pi, \mathcal{M})& =\mathbb{E}_{\tau\sim\pi}[\sum\nolimits_{i=0}^{n-1} \gamma^{i}r(s_i, a_i)],\\
J(\pi) & = \mathbb{E}_{\mathcal{M}\sim\mathcal{D}} [\eta(\pi, \mathcal{M})],\\ 
Q_{\pi}(s, a) & =\mathbb{E}_{\tau\sim\pi}[\sum\nolimits_{i=0}^{n-1}\gamma^i r(s_i, a_i)|s_0 = s, a_0=a], \\ 
V_{\pi}(s)& =\mathbb{E}_{\tau\sim\pi}[\sum\nolimits_{i=0}^{n-1}\gamma^i r(s_i, a_i)|s_0=s], \\
A_{\pi}(s, a) & = Q_{\pi}(s, a) - V_{\pi}(s),\\ 
\end{align*}
with $\eta$ being the expected cumulative reward of an individual MDP $\mathcal{M}$, $J$ the overall objective, $Q$ the Q function, $V$ the value function and $A$ the advantage function.
The performance of two policies can be related via the advantage function \citep{schulman2015trust, kakade2002approximately}: $
    \eta(\pi', \mathcal{M}) = \eta(\pi, \mathcal{M}) + \mathbb{E}_{\tau\sim\pi'}[\sum\nolimits_{i=0}^{n-1}\gamma^i A_{\pi}(s_i, a_i)]$.
Based on Theorem \ref{theorem:occupancy} below, we can rewrite the final term with the occupancy measure, $\rho_{\pi}(s, a) =\mathbb{P}(\pi(s)=a)\sum\nolimits_{i=0}^{\infty}\gamma^i \mathbb{P}(s_i=s|\pi)$.

\begin{theorem} 
\label{theorem:occupancy}
(Theorem 2 of \citep{syed2008apprenticeship}). For any policy $\pi$, it is the only policy that has its corresponding occupancy measure $\rho_{\pi}$, i.e., there is a one-to-one mapping between policies and occupancy measures. Specifically, $\mathbb{P}(\pi(s)=a)=\frac{\rho_{\pi}(s,a)}{\sum_{a'}\rho_{\pi}(s, a')}$.
\end{theorem}

With slight notation abuse, define $\rho_{\pi}(s)=\sum\nolimits_{i=0}^{\infty}\gamma^i \mathbb{P}(s_i=s|\pi)$ to be the state visitation distribution. In policy iteration, we aim to maximize:
\begin{align*}
    &\mathbb{E}_{\tau\sim\pi'}[\sum\nolimits_{i=0}^{n-1}\gamma^i A_{\pi}(s_i, a_i)], \\
    &= \sum\nolimits_{i=0}^{n-1}\expect{s_i\sim \rho_{\pi'}(s)}{\expect{a_i\sim \pi'(s_i)}{\gamma^i A_{\pi}(s_i, a_i)}},\\
    &\approx \sum\nolimits_{i=0}^{n-1}\expect{s_i\sim \rho_\pi(s)}{\expect{a_i\sim \pi'(s_i)}{\gamma^i A_{\pi}(s_i, a_i)}}.
\end{align*}
This is done instead of taking an expectation over  $\rho_{\pi'}(s)$ which has a complicated dependency on a yet unknown policy $\pi'$. Policy gradient methods tend to use the approximation by using $\rho_{\pi}$ which depends on the current policy. We define the approximate objective as:
\begin{align*}
&\eta_{\pi}(\pi', \mdp) \\
&= \eta(\pi, \mdp) + \sum\nolimits_{i=0}^{n-1}\expect{s_i\sim \rho_\pi(s)}{\expect{a_i\sim \pi'(s_i)}{\gamma^i A_{\pi}(s_i, a_i)}},
\end{align*}
and its associated expectation over $\dist$ as $J_{\pi}(\pi')=\expect{\mdp\sim\dist}{\eta_{\pi}(\pi', \mdp)}$.


\section{A THEORY OF POLICY CO-TRAINING}
\label{sec:theory}
In this section, we provide two theoretical characterizations of policy co-training. These characterizations highlight a trade-off in sharing information between different views, and motivates the design of our \copier algorithm presented in Section \ref{sec:algo}.

We restrict our analysis to infinite horizon MDPs, and thus require a strict discount factor $\gamma < 1$.  We show in our experiments that our \copier algorithm performs well even in finite horizon MDPs with $\gamma=1$.
Due to space constraints, we defer all proofs to the appendix. 

We present two theoretical analyses with different types of guarantees:
\begin{itemize}
    \item Section \ref{sec:theory:demo} quantifies the \textbf{policy improvement} in terms of policy advantages and differences, and caters to policy gradient approaches. 
    \item Section \ref{sec:theory:shared} quantifies the \textbf{performance gap} with respect to an optimal policy in terms of policy disagreements, which is a stronger guarantee than policy improvement.  This analysis is restricted to the shared action space setting, and caters to learning reduction approaches. 
\end{itemize}

\subsection{GENERAL CASE: POLICY IMPROVEMENT WITH DEMONSTRATIONS}
\label{sec:theory:demo}
For an MDP $\mathcal{M}\sim \mathcal{D}$, consider the rewards of two policies with different views $\eta^A(\pi^A, \mathcal{M}^A)$ and $\eta^B(\pi^B, \mathcal{M}^B)$. If $\eta^A(\pi^A, \mathcal{M}^A)> \eta^B(\pi^B, \mathcal{M}^B)$, $\pi^A$ performs better than $\pi^B$ on this instance , and we could use the translated trajectory of $\pi^A$ as a demonstration for $\pi^B$. Even when $J(\pi^A) > J(\pi^B)$, because $J$ is computed in expectation over $\mathcal{D}$, $\pi^B$ can still outperform $\pi^A$ on some MDPs. Thus it is possible for the exchange of demonstrations to go in both directions.

Formally, we can partition the distribution $\dist$ into two (unnormalized) parts $\dist_1$ and $\dist_2$ such that the support of $\dist$, $\supp{\dist} = \supp{\dist_1}\cup\supp{\dist_2}$ and $\supp{\dist_1}\cap\supp{\dist_2}=\emptyset$, where for an MDP $\mathcal{M}\in\supp{\dist_1}, \eta(\pi^A, \mdp^A) \ge \eta(\pi^B, \mdp^B)$ and for an MDP $\mathcal{M}\in\supp{\dist_2}, \eta(\pi^B, \mdp^B) > \eta(\pi^A, \mdp^A)$. 
By construction, we can quantify the performance gap as: 
\begin{definition}
\begin{align*}
    &\delta_1 = \expect{\mdp\sim\dist_1}{\eta(\pi^A, \mdp^A)-\eta(\pi^B, \mdp^B)} \ge 0,
    \\
    &\delta_2 = \expect{\mdp\sim\dist_2}{\eta(\pi^B, \mdp^B)-\eta(\pi^A, \mdp^A)}  > 0.
\end{align*}
\end{definition}
We can now state our first result on policy improvement.
\begin{theorem} 
\label{theorem:general}
(Extension of Theorem 1 in \citep{kang2018policy}) 
Define:
\begin{align*}
&\alpha^A_{\dist} = \expect{\mdp\sim\dist}{\max\nolimits_{s} D_{KL}(\pi^A(s)\| \pi'^{A}(s))},\\ &\beta^B_{\dist_2}=\expect{\mdp\sim\dist_2}{\max\nolimits_{s} D_{JS}(\pi^B(s)\|\pi^A(s))},\\
&\alpha^B_{\dist} = \expect{\mdp\sim\dist}{\max\nolimits_{s} D_{KL}(\pi^B(s)\| \pi'^{B}(s))},\\ &\beta^A_{\dist_1}=\expect{\mdp\sim\dist_1}{\max\nolimits_{s} D_{JS}(\pi^A(s)\|\pi^B(s))}, \\
    &\epsilon^B_{\dist_2}=\max\nolimits_{\mdp\in\supp{\dist_2}}\max\nolimits_{s,a}|A_{\pi^B}(s, a)|,\\
    &\epsilon^{A}_{\dist}=\max\nolimits_{\mdp\in\supp{\dist}}\max\nolimits_{s,a}|A_{\pi^A}(s, a)|,\\
    &\epsilon^A_{\dist_1}=\max\nolimits_{\mdp\in\supp{\dist_1}}\max\nolimits_{s,a}|A_{\pi^A}(s, a)|,\\
    &\epsilon^{B}_{\dist}=\max\nolimits_{\mdp\in\supp{\dist}}\max\nolimits_{s,a}|A_{\pi^B}(s, a)|.
\end{align*}
Here $D_{KL}$ \& $D_{JS}$ denote the Kullback-Leibler and Jensen-Shannon divergence respectively. Then we have:
\begin{small}
\begin{align*}
    &J(\pi'^A) \ge J_{\pi^A}(\pi'^A) - \frac{2\gamma^A(4\beta^B_{\dist_2} \epsilon^B_{\dist_2} + \alpha^A_{\dist} \epsilon^{A}_{\dist})}{(1-\gamma^A)^2} + \delta_2, \\
    &J(\pi'^B) \ge J_{\pi^B}(\pi'^B) - \frac{2\gamma^B(4\beta^A_{\dist_1} \epsilon^A_{\dist_1} + \alpha^B_{\dist} \epsilon^{B}_{\dist})}{(1-\gamma^B)^2} + \delta_1.
\end{align*}
\end{small}
\end{theorem}
Compared to conventional analyses on policy improvement, the new key terms that determine how much the policy improves are the $\beta$'s and $\delta$'s. The $\beta$'s, which quantify the maximal divergence between $\pi^A$ and $\pi^B$, hinders improvement, while the $\delta$'s contribute positively. If the net contribution is positive, then the policy improvement bound is larger than that of conventional single view policy gradient. This insight motivates co-training algorithms that explicitly aim to minimize the $\beta$'s.



One technicality is how to compute $D_{JS}(\pi^A(s)\|\pi^B(s))$ given that the state and action spaces for the two representations might be different. Proposition \ref{proposition:divergence} ensures that we can measure the Jensen-Shannon divergence between two policies with different MDP representations.
\begin{proposition}
\label{proposition:divergence}
For representations $\mdp^A$ and $\mdp^B$ of an MDP satisfying  Assumption \ref{assumption:mapping}, the quantities $\max\nolimits_{s} D_{JS}(\pi^B(s)\|\pi^A(s))$ and $\max\nolimits_{s} D_{JS}(\pi^A(s)\|\pi^B(s))$ are well-defined.
\end{proposition}

Minimizing $\beta^B_{\dist_2}$ and $\beta^A_{\dist_1}$ is not straightforward since the trajectory mappings between the views can be very complicated. We present practical algorithms in Section \ref{sec:algo}.

\subsection{SPECIAL CASE: PERFORMANCE GAP FROM OPTIMAL POLICY IN SHARED ACTION SETTING}
\label{sec:theory:shared}

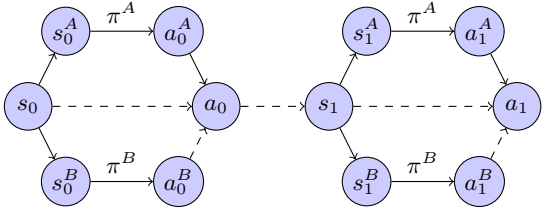
\begin{figure}
\centering
\begin{tikzpicture}[main_node/.style={circle,fill=blue!20,draw,minimum size=2em,inner sep=1.5pt}]
    \footnotesize 
    \node[main_node] (1) at (0, 0) {$s_0$};
    \node[main_node] (2) at (0.5, 1)  {$s_0^A$};
    \node[main_node] (3) at (0.5, -1) {$s_0^B$};
    
    \node[main_node] (4) at (2, 1) {$a_0^A$};
    \node[main_node] (5) at (2, -1)  {$a_0^B$};
    \node[main_node] (6) at (2.5, 0) {$a_0$};
    
    \draw [->] (1) -- (2);
    \draw [->] (1) -- (3);
    
    \draw [->] (2) -- (4) node [midway, above] {$\pi^A$};
    \draw [->] (3) -- (5) node [midway, above] {$\pi^B$};
    \draw [->] (4) -- (6);
    \draw [dashed, ->] (5) -- (6);  
    
    \draw [dashed, ->] (1) -- (6);
    
    \node[main_node] (7) at (4, 0) {$s_1$};
    \node[main_node] (8) at (4.5, 1)  {$s_1^A$};
    \node[main_node] (9) at (4.5, -1) {$s_1^B$};
    
    \node[main_node] (10) at (6, 1) {$a_1^A$};
    \node[main_node] (11) at (6, -1)  {$a_1^B$};
    \node[main_node] (12) at (6.5, 0) {$a_1$};
    
    \draw [dashed, ->] (6) -- (7);
    
    \draw [->] (7) -- (8);
    \draw [->] (7) -- (9);
    
    \draw [->] (8) -- (10) node [midway, above] {$\pi^A$};
    \draw [->] (9) -- (11) node [midway, above] {$\pi^B$};
    
    \draw [->] (10) -- (12);
    \draw [dashed, ->] (11) -- (12);  
    
    \draw [dashed, ->] (7) -- (12);
    \end{tikzpicture}
    \caption{Co-training with shared action space.}
    \label{fig:share_action}
\end{figure}

We now analyze the special case where the action spaces of the two views are the same, i.e., $\mathcal{A}^A = \mathcal{A}^B$. Figure \ref{fig:share_action} depicts the learning interaction between $\pi^A$ and $\pi^B$. For each state $s$, we can directly compare actions chosen by the two policies since the action space is the same. This insight leads to a stronger analysis result where we can bound the gap between a co-trained policy with an optimal policy. The approach we take resembles learning reduction analyses for interactive imitation learning.

For this analysis we focus on discrete action spaces with $k$ actions, deterministic learned policies, and a deterministic optimal policy (which is guaranteed to exist \citep{puterman2014markov}). 
We reduce policy learning to classification: for a given state $s$, the task of identifying the optimal 
action $\pi^*(s)$ is a classification problem. We build upon the PAC generalization bound results in \citep{dasgupta2002pac} and show that under Assumption \ref{assumption:independence}, optimizing a measure of disagreements between the two policies leads to effective learning of $\pi^*$. 
\begin{assumption}
\label{assumption:independence} For a state $s$, its two representations $s^A$ and $s^B$ are conditionally independent given the optimal action $\pi^*(s)$.
\end{assumption}
This assumption is common in analyses of co-training for classification \citep{blum1998combining, dasgupta2002pac}. Although this assumption is typically violated in practice \citep{nigam2000analyzing}, our empirical evaluation still demonstrates strong performance. 

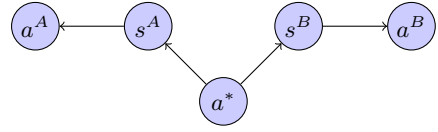
\begin{figure}
\centering
\begin{tikzpicture}[main_node/.style={circle,fill=blue!20,draw,minimum size=2em,inner sep=1pt}]
    \footnotesize
    \node[main_node] at (0, 0)  (a1)    {$a^A$};
    \node[main_node] at (1.5, 0)  (s1)    {$s^A$};

    \node[main_node] at (5, 0)  (a2)    {$a^B$};
    \node[main_node] at (3.5, 0)  (s2)    {$s^B$};
    
    \node[main_node] at (2.5, -1) (a)     {$a^*$};
    
    \draw [->] (a) -- (s1);
    \draw [->] (a) -- (s2);
    \draw [->] (s1) -- (a1);
    \draw [->] (s2) -- (a2);
\end{tikzpicture}
\caption{Graphical model encodes the conditional independence model.}
\label{fig:model}
\end{figure}

Assumption \ref{assumption:independence} corresponds to a graphical model describing the relationship between optimal actions and the state representations (Figure \ref{fig:model}).
The intuition is that, when we do not know $a^*=\pi^*(s)$, we should maximize the agreement between $a^A=\pi^A(s^A)$ and $a^B=\pi^B(s^B)$. By the data-processing inequality in information theory \citep{cover2012elements}, we know that $\mathbb{I}(a^A; a^*) \ge \mathbb{I}(a^A; a^B)$. In practice, this means that if $a^A$ and $a^B$ agree a lot, they must reveal substantial information about what $a^*$ is. We formalize this intuition and obtain an upper bound on the classification error rate, which enables quantifying the performance gap. Notice that if we do not have any information from $\pi^*$, the best we can hope for is to learn a mapping that matches $\pi^*$ up to some permutation of the action labels \citep{dasgupta2002pac}. Thus we assume we have enough state-action pairs from $\pi^*$ so that we can recover the permutation. In practice this is satisfied as we demonstrate in Section \ref{sec:sec:control}. 

Formally, we connect the performance gap between a learned policy and an optimal policy with an empirical estimation on the disagreement in action choices among two co-trained policies. Let $\{\tau^A_i\}_{i=1}^{m}$ be sampled trajectories from $\pi^A$ and $\{f_{A\rightarrow B}(\tau^A_i)\}_{i=1}^{m}$ be the mapped trajectories in $\mdp^B$.  In $\{f_{A\rightarrow B}(\tau^A_i)\}_{i=1}^{m}$, let $N(a^A=i)$ be the number of times action $i$ is chosen by $\pi^A$ and $N=\sum\nolimits_{i=1}^k N(a^A=i)$ be the total number of actions in one trajectory set. 
Let $N(a^B=i)$ be the number of times action $i$ is chosen by $\pi^B$ when going through the states in $\{f_{A\rightarrow B}(\tau^A_i)\}_{i=1}^{m}$ and $N(a^A=i, a^B=i)$ record when both actions agree on $i$.

We also require a measure of model complexity, as is common in PAC style analysis. We use $|\pi|$ to denote the number of bits needed to represent $\pi$. We can now state our second main result quantifying the performance gap with respect to an optimal policy:

\begin{theorem}
\label{theorem:special}
If Assumption \ref{assumption:independence} holds for $\mdp\sim\dist$ and a deterministic optimal policy $\pi^*$. Let $\pi^A$ and $\pi^B$ be two deterministic policies for the two representations. \\
Define:
\begin{align*}
&\hat{\mathbb{P}}(a^A=i \mid a^B=i) = \frac{N(a^A=i, a^B=i)}{N(a^B=i)}, \\
&\hat{\mathbb{P}}(a^A\ne i \mid a^B=i) = \frac{N(a^A\ne i, a^B=i)}{N(a^B=i)}, \\
&\epsilon_i(\pi^A, \pi^B,\sigma) = \sqrt{\frac{\ln{2}(|\pi^A|+|\pi^B|)+\ln{(2k/\sigma)}}{2N(a^B = i)},} \\
&\zeta_i(\pi^A, \pi^B, \sigma) = \hat{\mathbb{P}}(a^A=i\mid a^B=i) \\ 
&\hspace{2.3cm}  - \hat{\mathbb{P}}(a^A\ne i \mid a^B=i) - 2\epsilon_i(\pi^A, \pi^B, \sigma),
\end{align*}
\begin{align*}
&b_i(\pi^A, \pi^B, \sigma) = \frac{1}{\zeta_i(\pi^A, \pi^B, \delta)}(\hat{\mathbb{P}}(a^A\ne i\mid a^B=i) \\ 
&\hspace{2.4cm} +\epsilon_i(\pi^A,\pi^B,\sigma)), \\
&\ell(s, \pi) = \textbf{1}(\pi(s) \ne \pi^*(s)), \\
&\epsilon_A=\expect{s\sim\rho_{\pi^A}}{\ell(s, \pi^A)}.
\end{align*}
Then with probability $1-\sigma$:
\begin{align*}
\epsilon_A & \le \max\nolimits_{j\in\{1, \cdots, k\}}b_j(\pi^A, \pi^B, \sigma),\\
\eta(\pi^A, \mdp^A) &\ge \eta(\pi^*, \mdp) - uT\epsilon_A,
\end{align*}
where $T$ is the time horizon and $u$ is the largest one-step deviation loss compared with $\pi^*$.
\end{theorem}
To obtain a small performance gap compared to $\pi^*$, one must minimize $\epsilon_A$, which measures the disagreement between $\pi^A$ and $\pi^*$. However, we cannot directly estimate this quantity since we only have limited sample trajectories from $\pi^*$. Alternatively, we can minimize an upper bound, $\max\nolimits_{j\in\{1, \cdots, k\}}b_j(\pi^A, \pi^B, \delta)$, which measures the maximum disagreement on actions between $\pi^A$ and $\pi^B$ and, importantly, can be estimated via samples. In Section \ref{sec:algo:special}, we design an algorithm that approximately minimizes this bound. The advantage of two views over a single view enables us to establish an upper bound on $\epsilon_A$, which is otherwise unmeasureable.

\section{THE \copier ALGORITHM}
\label{sec:algo}
We now present practical algorithms motivated by the theoretical insights from Section \ref{sec:theory}. We start with a meta-algorithm named \copier (Algorithm \ref{alg:copier}), whose important subroutines are \exchange and \update. We provide two concrete instantiations  for the general case and the special case with a shared action space. 

\begin{algorithm}[H]
   \caption{\copier (Co-training for Policy Learning)}
   \label{alg:copier}
\begin{algorithmic}[1]
   \STATE {\bfseries Input:} A distribution $\dist$ of MDPs, two policies $\pi^A, \pi^B$, mapping functions $f_{A\rightarrow B}, f_{B\rightarrow A}$
   \REPEAT
   \STATE Sample $\mdp\sim\dist$, form $\mdp^A, \mdp^B$
   \STATE Run $\pi^A$ on $\mdp^A$ to generate trajectories $\{\tau^A_i\}_{i=1}^{m}$
   \STATE Run $\pi^B$ on $\mdp^B$ to generate trajectories $\{\tau^B_j\}_{j=1}^{n}$   
   \STATE $\{{\tau'_i}^A\},\{{\tau'_j}^B\} \leftarrow  \exchange(\{\tau^A_i\}, \{\tau^B_j\})$
   \STATE $\pi^A\leftarrow \update(\pi^A, \{\tau^A_i\}, \{{\tau'_j}^A\})$
   \STATE $\pi^B\leftarrow \update(\pi^B, \{\tau^B_i\}, \{{\tau'_j}^B\})$
   \UNTIL{Convergence}
\end{algorithmic}
\end{algorithm}

\subsection{GENERAL CASE}
\label{sec:algo:general}
Algorithm \ref{alg:exchange:general} covers the general case for exchanging trajectories generated by the two policies. First we estimate the relative quality of the two policies from sampled trajectories (Lines 2-4 in Algorithm \ref{alg:exchange:general}). Then we use the trajectories from the better policy as demonstrations for the worse policy on this MDP. This mirrors the theoretical insight presented in Section \ref{sec:theory}, where based on which sub-distribution an MDP is sampled from, the relative quality of the two policies is different.

For \update, we can form a loss function that is derived from either imitation learning or reinforcement learning.  
Recall that we aim to optimize the $\beta$ terms in Theorem \ref{theorem:general}, however it is infeasible to directly optimize them. So we consider a surrogate loss $C$ (Line 2 of Algorithm \ref{alg:update}) that measures the policy difference. In practice, we typically use behavior cloning loss as the surrogate.

\begin{algorithm}[t]
   \caption{\exchange: General Case}
   \label{alg:exchange:general}
\begin{algorithmic}[1]
   \STATE {\bfseries Input:} Trajectories $\{\tau^A_i\}_{i=1}^m$ and $\{\tau^B_j\}_{j=1}^n$ 
   \STATE Compute estimate $\hat{\eta}(\pi^A, \mdp^A) = \frac{1}{m}\sum\nolimits_{i=1}^m r(\tau^A_i)$
   \STATE Compute estimate $\hat{\eta}(\pi^B, \mdp^B) = \frac{1}{n}\sum\nolimits_{j=1}^n r(\tau^B_j)$
   \IF{$\hat{\eta}(\pi^A, \mdp^A) > \hat{\eta}(\pi^B, \mdp^B)$}
   \STATE $\{\tau^{A\rightarrow B}_i\} \leftarrow \{f_{A\rightarrow B}(\tau^A_i)\}_{i=1}^m$
   \STATE $\{\tau^{B\rightarrow A}_j\} \leftarrow \emptyset$   
   \ELSE 
   \STATE $\{\tau^{A\rightarrow B}_i\} \leftarrow \emptyset$ 
   \STATE $\{\tau^{B\rightarrow A}_j\} \leftarrow \{f_{B\rightarrow A}(\tau^B_j)\}_{j=1}^n$  
   \ENDIF
   \STATE \textbf{return} {$\{\tau^{A\rightarrow B}_i\}, \{\tau^{B\rightarrow A}_j\}$}
\end{algorithmic}
\end{algorithm}

\begin{algorithm}[t]
   \caption{\update}
   \label{alg:update}
\begin{algorithmic}[1]
   \STATE {\bfseries Input:} Current policy $\pi$, sampled trajectories from $\pi$, $\{\tau_i\}_{i=1}^m$ and demonstrations $\{\tau'_j\}_{j=1}^n$ 
   \STATE Form a loss function $\mathcal{L}(\pi) = $
   $
   \begin{cases}
   -\sum\nolimits_{i=1}^m r(\tau_i) + \lambda C(\pi, \{\tau'_j\}_{j=1}^n), \text{ RL with IL loss} \\
   \lambda C(\pi, \{\tau'_j\}_{j=1}^n), \text{ IL loss only}
   \end{cases}$
   \STATE Update $\pi \leftarrow \pi - \alpha \nabla \mathcal{L}(\pi)$
\end{algorithmic}
\end{algorithm}
\begin{algorithm}[t]
   \caption{\exchange: Special Case}
   \label{alg:exchange:special}
\begin{algorithmic}[1]
   \STATE {\bfseries Input:} Trajectories $\{\tau^A_i\}_{i=1}^m$ and $\{\tau^B_j\}_{j=1}^n$ 
   \STATE $D^{A\rightarrow B} = \interactive(\{f_{B\rightarrow A}(\tau^B_j)\}_{j=1}^n, \pi^A)$
   \STATE $D^{B\rightarrow A} = \interactive(\{f_{A\rightarrow B}(\tau^A_i)\}_{i=1}^m, \pi^B)$
   \STATE \textbf{return} $D^{A\rightarrow B}, D^{B\rightarrow A}$
\end{algorithmic}
\end{algorithm}

\begin{algorithm}[t]
   \caption{\interactive}
   \label{alg:exchange:interactive}
\begin{algorithmic}[1]
   \STATE {\bfseries Input:} Trajectories $\{\tau_i\}_{i=1}^m$, query policy $\pi$ 
   \STATE $D = \emptyset$
   \FOR{$i\leftarrow 1 \text{ to } m$} 
   \FOR{each state $s \in \tau_i$}
   \STATE $D \leftarrow D\cup\{(s, \pi(s))\}$
   \ENDFOR
   \ENDFOR
   \STATE \textbf{return} $D$
\end{algorithmic}
\end{algorithm}
\subsection{SPECIAL CASE: SHARED ACTION SPACE}
\label{sec:algo:special}
For the special case with a shared action space, we can collect more informative feedback beyond the trajectory level. Instead, we collect interactive state-level feedback, as is popular in imitation learning algorithms such as DAgger \citep{ross2011reduction} and related approaches \cite{sun2017deeply,daume2009search,ross2010efficient,song2018learning,he2014learning}. Specifically, we can use Algorithms \ref{alg:exchange:special} \& \ref{alg:exchange:interactive} to exchange actions in a state-coupled manner. This process is depicted in Figure \ref{fig:share_action}, where $\pi^A$'s visited states, $s_0^A$ and $s_1^A$, are mapped to $s^B_0$ and $s^B_1$, resulting in receiving $\pi^B$'s actions, $a_0^B$ and $a_1^B$, in the exchange. 

Unlike the general case where information exchange is asymmetric, as Theorem \ref{theorem:special} indicates, we aim to minimize policy disagreement. Both policies are simultaneously optimizing this objective, which requires both directions of information exchange (Lines 2-3 in Algorithm \ref{alg:exchange:special}).
The update step (Algorithm \ref{alg:update}) is the same as the general case.

\section{EXPERIMENTS}
We now present empirical results on both the special and general cases of \copier. We demonstrate the generality of our approach by applying three distinct combinations of policy co-training: reinforcement learning on both views (Section \ref{sec:sec:control}), reinforcement learning on one view and imitation learning on the other (Section \ref{sec:sec:mvc}), and imitation learning on both views (Section \ref{sec:sec:psulu}). Furthermore, our experiments on combinatorial optimization (Sections \ref{sec:sec:mvc} \& \ref{sec:sec:psulu}) demonstrate significant improvements over strong learning-based baselines as well as commercial solvers, and thus showcase the practicality of our approach.  More details about the experiment setup can be found in the appendix.
\label{sec:exp}

\subsection{DISCRETE \& CONTINUOUS CONTROL: SPECIAL CASE WITH RL+RL}
\label{sec:sec:control}
\begin{figure*}[t]
  \centering
  \begin{subfigure}[b]{.32\textwidth}
    \centering
    {
      \captionsetup{width=.9\linewidth}
      \includegraphics[trim={0pt 0pt 0pt 0pt}, width=\textwidth]{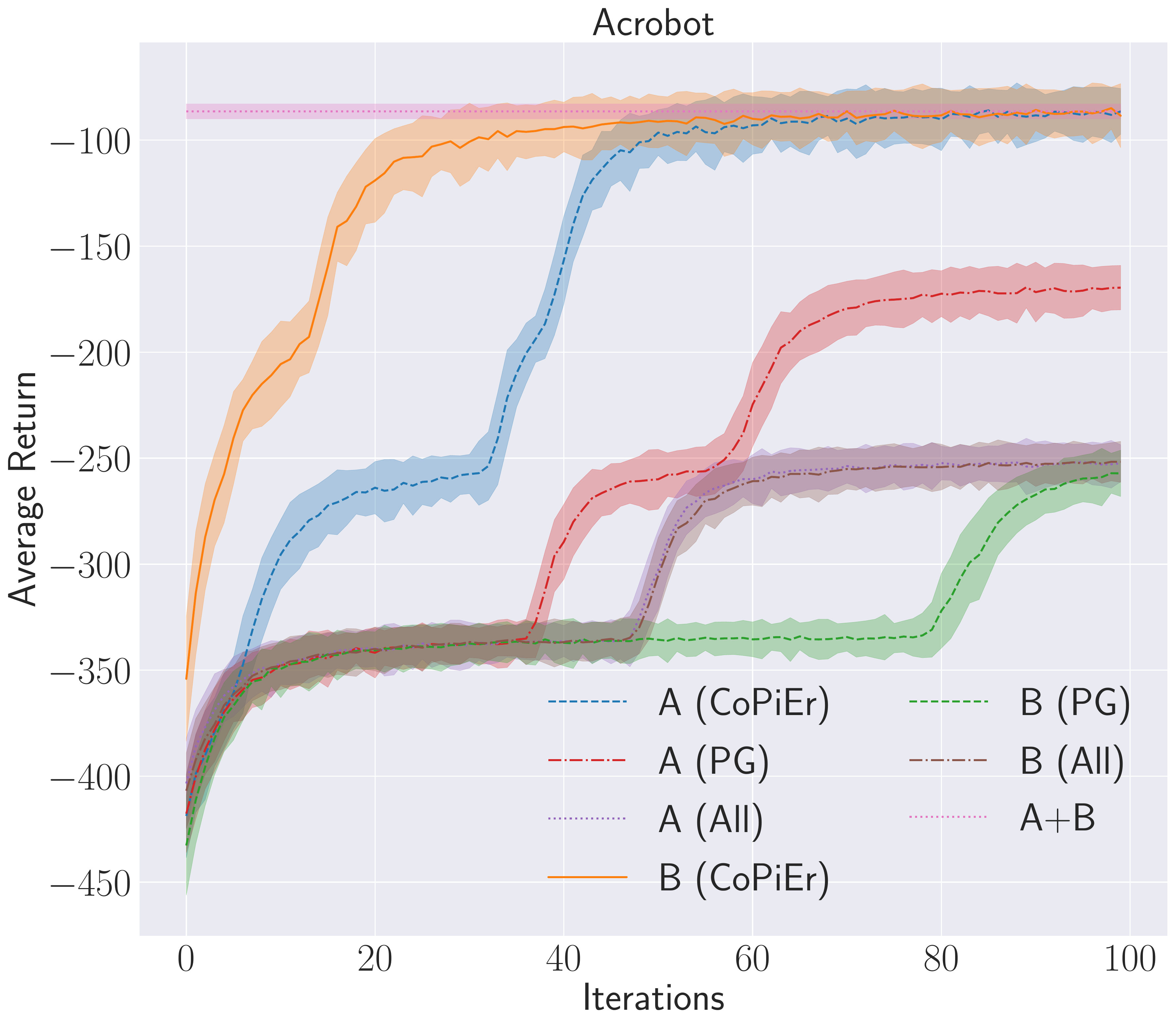}
      \caption{Acrobot Swing-up. A denotes removing the first coordinate in the state vector and B removing the second coordinate.}
      \label{fig:exp:acrobot}
    }
  \end{subfigure}
  \begin{subfigure}[b]{.32\textwidth}
    \centering
    {
            \captionsetup{width=.9\linewidth}
      \includegraphics[trim={0pt 0pt 0pt 0pt}, width=\textwidth]{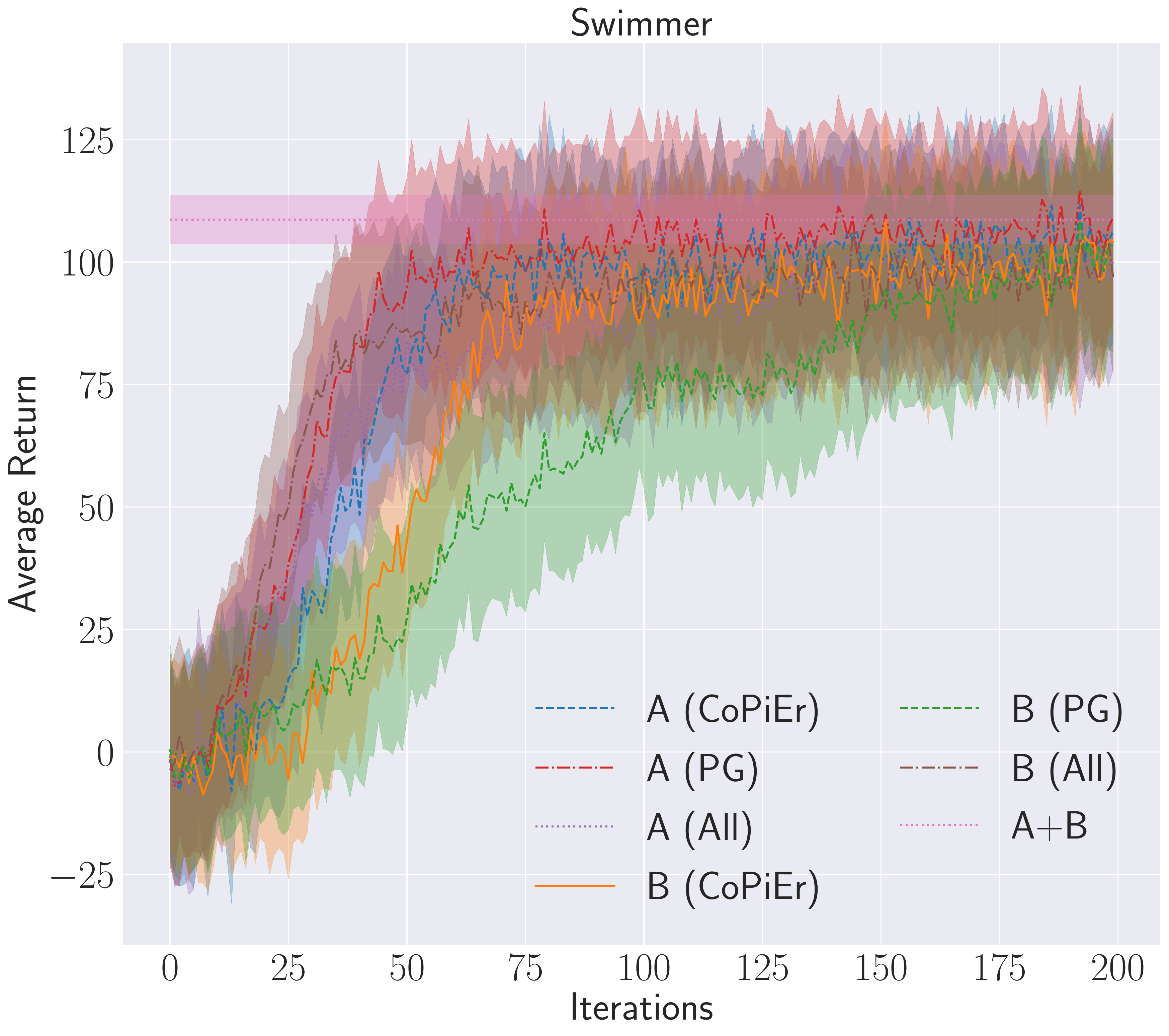}
      \caption{Swimmer. A denotes removing all even index coordinates in the state vector and B removing all odd index ones.}
      \label{fig:exp:swimmer}
    }
  \end{subfigure}
  \begin{subfigure}[b]{.32\textwidth}
    \centering
    {
          \captionsetup{width=.9\linewidth}
      \includegraphics[trim={0pt 0pt 0pt 0pt}, width=\textwidth]{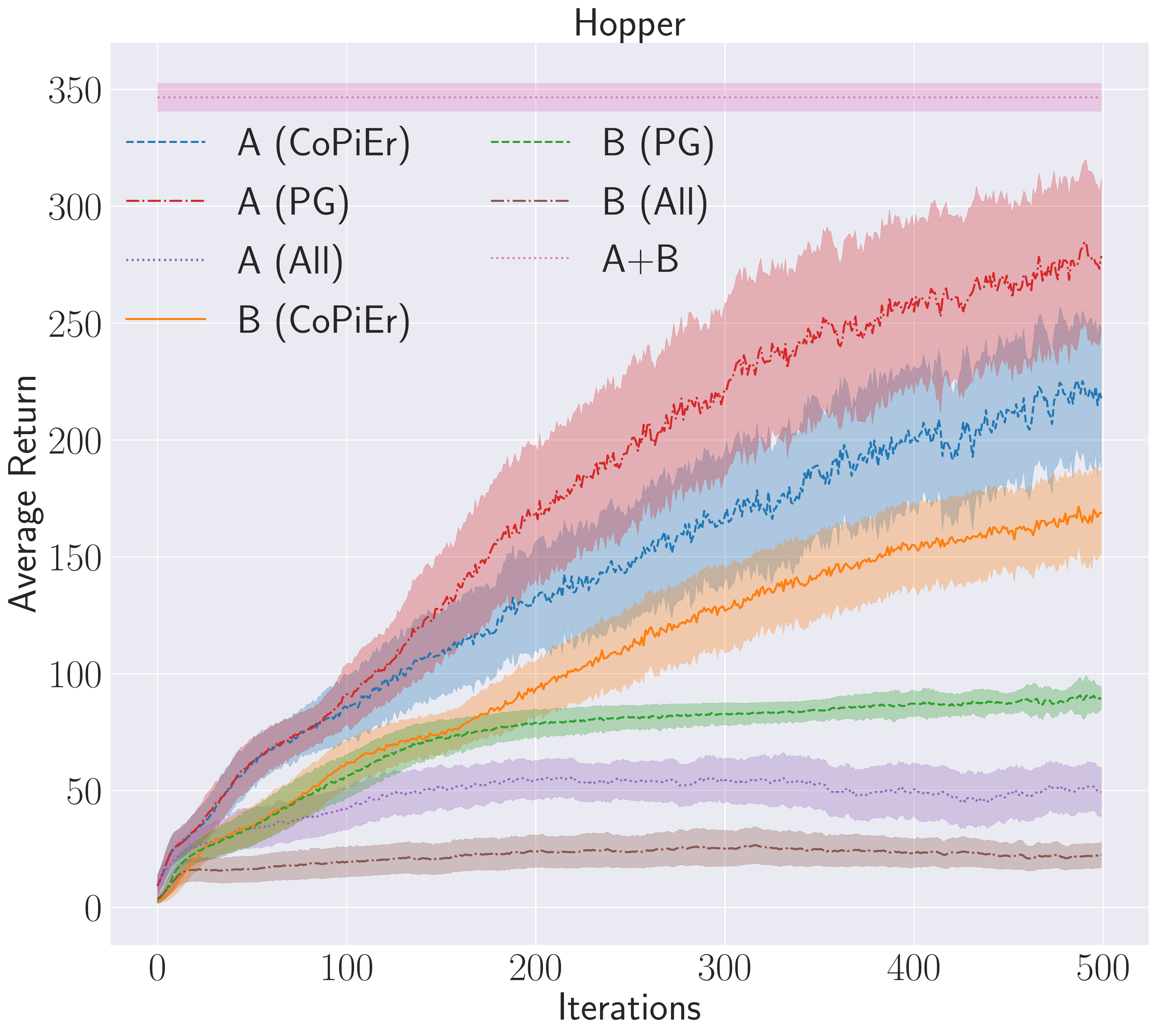}
      \caption{Hopper. A denotes removing all even index coordinates in the state vector and B removing all odd index ones.}
      \label{fig:exp:hopper}
    }
  \end{subfigure}
    \caption{Discrete \& continuous control tasks. Experiment results are across 5 random seeded runs. Shaded area indicates $\pm 1$ standard deviation.}
    \label{fig:exp:control}
\end{figure*}
\textbf{Setup.}
We conduct experiments on discrete and continuous control tasks with OpenAI Gym \citep{brockman2016openai} and Mujoco physical engine \citep{todorov2012mujoco}. We use the garage repository \citep{duan2016benchmarking} to run reinforcement learning for both views.

\textbf{Two Views and Features.} For each environment, states are represented by feature vectors, typically capturing location, velocity and acceleration. We create two views by removing different subsets of features from the complete feature set. Note that both views have the same underlying action space as the original MDP, so it is the special case covered in Section \ref{sec:algo:special}. We use interactive feedback for policy optimization.

\textbf{Policy Class.} We use a feed-forward neural network with two hidden layers (64 \& 32 units) and tanh activations as the policy class. For discrete actions, $\pi(s)$ outputs a soft-max distribution. For continuous actions, $\pi(s)$ outputs a (multivariate) Gaussian. For policy update, we use Policy Gradient \citep{sutton2000policy} with a linear baseline function \citep{greensmith2004variance} and define the loss function $C$ in Algorithm \ref{alg:update} to be the KL-divergence between output action distributions.

\textbf{Methods Compared.} We compare with single view policy gradient, labelled as ``A (PG)" and ``B (PG)", and with a policy trained on the union of the two views but test on two views separately, labelled as ``A (All)" and ``B (All)". We also establish an upper bound on performance by training a model without view splitting (``A+B"). Each method uses the same total number of samples (i.e., \copier uses half per view).

\textbf{Results.}  Figure \ref{fig:exp:control} shows the results. \copier is able to converge to better or comparable solutions in almost all cases except for view A in Hopper. The poor performance in Hopper could be due to the disagreement between the two policies not shrinking enough to make Theorem \ref{theorem:special} meaningful. As a comparison, at end of the training, the average KL-divergence for the two policies is about 2 for Hopper, compared with 0.23 for Swimmer and 0.008 for Acrobot. One possible cause for such large disagreement is that the two views have significance differences in difficulty for learning, which is the case for Hopper by noticing A (PG) and B (PG) have a difference in returns of about 190. 


\subsection{MINIMUM VERTEX COVER: GENERAL CASE WITH RL+IL}
\label{sec:sec:mvc}
\begin{figure}[t]
    \centering
    \includegraphics[trim={0pt 0pt 0pt 0pt}, width=0.485\textwidth]{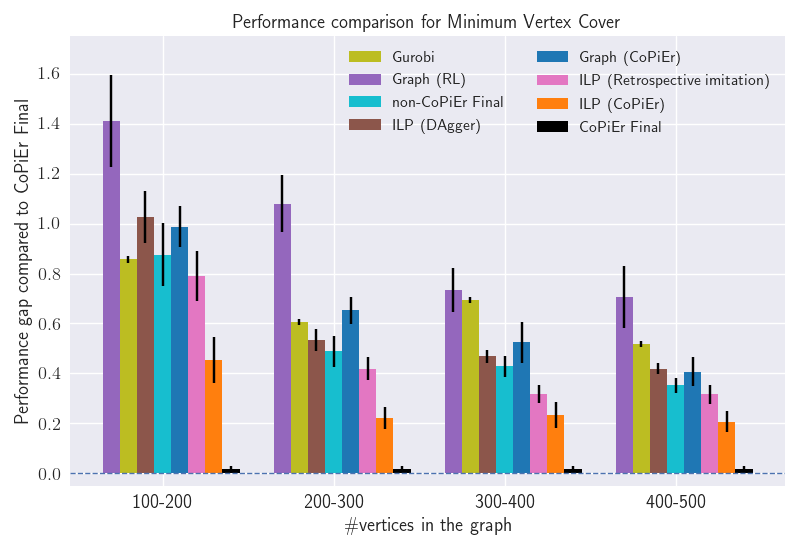}
    \vspace{-0.2in}
    \caption{Comparison of \copier with other learning-based baselines and a commercial solver, Gurobi. The $y$-axis measure relative gaps of various methods compared with \copier Final. \copier Final outperforms all the baselines. Notably, the gaps are significant because getting optimizing over large graphs is very challenging.}
    \label{fig:exp:mvc}
\end{figure}
\textbf{Setup.}
We now consider the challenging combinatorial optimization problem of minimum vertex cover (MVC). We use 150 randomly generated Erd\H{o}s-R\'enyi \citep{erdHos1960evolution} graph instances for each scale, with scales ranging \{100-200, 200-300, 300-400, 400-500\} vertices. For training, we use 75 instances, which we partition into 15 labeled and 60 unlabeled instances. We use the best solution found by Gurobi within 1 hour as the expert solution for the labeled set to bootstrap imitation learning. For each scale, we use 30 held-out graph instances for validation, and we report the performance on 45 test graph instances.

\textbf{Views and Features.} The two views are the graphs themselves and integer linear programs constructed from the graphs. For the graph view, we use DQN-based reinforcement learning \citep{dai2017learning} to learn a sequential vertex selection policy. We use {\tt structure2vec} \citep{Dai2016} to compute graph embeddings to use as state representations. For the ILP, we use imitation learning \citep{he2014learning} to learn node selection policy for branch-and-bound search. A node selection policy determines which node to explore next in the current branch-and-bound search tree. We use node-specific features (e.g., LP relaxation lower bound and objective value) and tree-specific features (e.g., integrality gap, and global lower and upper bounds) as our state representations. Vertex selection in graphs and node selection in branch-and-bound are different. So we use the general case algorithm in Section \ref{sec:algo:general}.

\textbf{Policy Class.}
For the graph view, our policy class is similar to \citep{dai2017learning}. In order to perform end-to-end learning of the parameters with labeled data exchanged between the two views, we use DQN \citep{mnih2013playing} with supervised losses \citep{Hester2018} to learn to imitate better demonstrations from the ILP view. For all our experiments, we determined the regularizer for the supervised losses and other parameters through cross-validation on the smallest scale (100-200 vertices). 
The graph view models are pre-trained with the labeled set using behavior cloning. We use the same number of training iterations for all the methods.

For the ILP view, our policy class consists of a node ranking model that prioritizes which node to visit next. We use RankNet \citep{Burges1998} as the ranking model, instantiated using a 2-layer neural network with ReLU as activation functions. We implement our approach for the ILP view within the SCIP \citep{Achterberg2009} integer programming framework.

\textbf{Methods Compared.}
At test time, when a new graph is given, we run both policies and return the better solution. We term this practical version ``\copier Final" and measure other policies' performance against it. We compare with single view learning baselines. For the graph view, we compare with RL-based policy learning over graphs \citep{dai2017learning}, labelled as ``Graph (RL)". And for the ILP view, we compare with imitation learning \citep{he2014learning} ``ILP (DAgger)", retrospective imitation \citep{song2018learning} ``ILP (Retrospective Imitation)" and a commercial solver Gurobi \citep{gurobi}.  We combine ``Graph (RL)" and ``ILP (DAgger)" as non-CoPiEr (Final) by returning the better solution of the two. We also show the performance of the two policies in \copier as standalone policies instead of combining them, labelled ``Graph (\copier)" and ``ILP (\copier)". ILP methods are limited by the same node budget in branch-and-bound trees.

\textbf{Results.}
Figure \ref{fig:exp:mvc} shows the results. We see that \copier Final outperforms all baselines as well as Gurobi. Interestingly, it also performs much better than either standalone \copier policies, which suggests that Graph (\copier) is better for some instances while ILP (\copier) is better on others. This finding validates combining the two views to maximize the benefits from both. For the exact numbers on the final performance, please refer to Appendix \ref{sec:control:tabular}.

\subsection{RISK-AWARE PATH PLANNING: GENERAL CASE WITH IL+IL}
\label{sec:sec:psulu}

\textbf{Setup.} We finally consider a practical application of risk-aware path planning \citep{ono2008efficient}. Given a start point, a goal point, a set of polygonal obstacles, and an upper bound of the probability of failure (risk bound), we must find a path, represented by a sequence of way points, that minimizes cost while limiting the probability of collision to within the risk bound. Details on the data generation can be found in the Appendix ~\ref{sec:data_gen_psulu}. We report the performance evaluations on 50 test instances.

\textbf{Views and Features.} This problem can be formulated into a mixed integer linear program (MILP) as well as a quadratically constrained quadratic program (QCQP), both of which can be solved using branch-and-bound  \citep{land2010automatic, linderoth2005simplicial}. For each view, we learn a node selection policy for branch-and-bound via imitation learning. Feature representations are similar to ILP view in MVC experiment (Section \ref{sec:sec:mvc}). For the QCQP view, we use the state variables bounds along the trace for each node from the root in the branch and bound tree as an additional feature.
Although the search framework is the same, because of the different nature of the optimization problem formulations, the state and action space are incompatible, and so we use the general case of \copier. A pictorial representation of the two views is presented in Appendix \ref{sec:pic_rep_psulu}. 

\textbf{Policy Class.}
The policy class for both MILP and QCQP views is similar to that of ILP view in MVC (Section \ref{sec:sec:mvc}), and we learn node ranking models.

\textbf{Methods Compared.}
Similar to MVC experiment, we compare other methods with ``\copier Final" which returns the better solution of the two. We use single view learning baselines, specifically those based on imitation learning \citep{he2014learning}, ``QCQP (DAgger)" and ``MILP(DAgger)", and on retrospective imitation \citep{song2018learning}, ``QCQP (Retrospective Imitation)" and ``MILP (Retrospective Imitation)". Two versions of non-CoPiEr Final are presented, based on DAgger and Retrospective Imitation, respectively. Gurobi is also used to solve MILPs but it is not able to solve the QCQPs because they are non-convex.

\begin{figure}[t]
    \centering
    \includegraphics[trim={0pt 0pt 0pt 0pt}, width=0.485\textwidth]{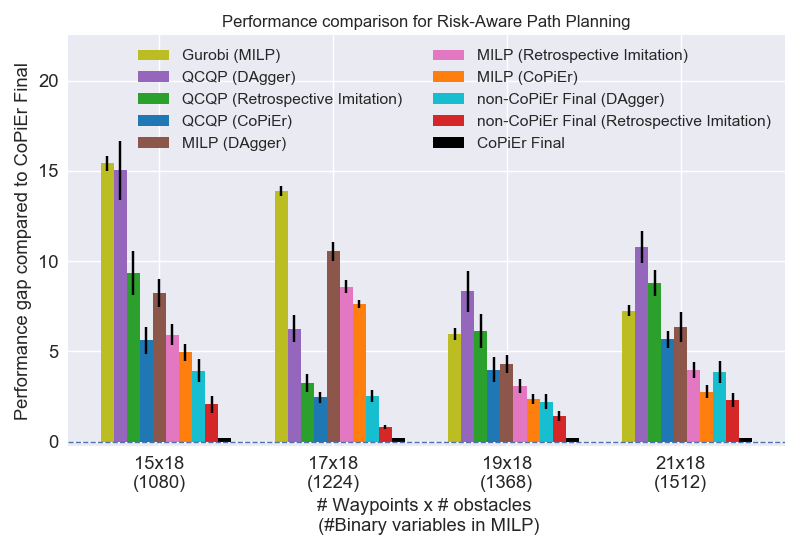}
    \vspace{-0.15in}
    \caption{Comparison of \copier with other learning-based baselines and a commercial solver, Gurobi. The $y$-axis measure relative gaps of various methods compared with \copier Final. \copier Final outperforms all the baselines. Notably, the scale of problems as measured by the number of integer variables far exceed previous state-of-the-art method \citep{song2018learning}.}
    \label{fig:exp:psulu}
\end{figure}

\textbf{Results.} Figure \ref{fig:exp:psulu} shows the results. Like in MVC, we again see that \copier Final outperforms  baselines as well as Gurobi. We also observe a similar benefit of aggregating both policies. The effectiveness of \copier enables solving much larger problems than considered in previous work \citep{song2018learning} (560 vs 1512  binary variables).

\section{CONCLUSION \& FUTURE WORK}
\label{sec:conclusion}
We have presented \copier (Co-training for Policy Learning), a general framework for policy learning for sequential decision making tasks with two representations. Our theoretical analyses and algorithm design cover both the general case as well as a special case with shared action spaces. 
Our approach is compatible with both reinforcement learning and imitation learning as subroutines.
We evaluated on a variety of settings, including control and  combinatorial optimization. Our results on showcase the generality of our framework and significant improvements over numerous baselines.

There are many interesting directions for future work.
On the theory front, directions include weakening assumptions such as  conditional independence, or extending to more than two views. On the application front, algorithms such as \copier can potentially improve performance in a wide range of  robotic and other autonomous systems that utilize  different sensors and image data. 

\smallskip
\textbf{Acknowledgments.} The work was funded in part by NSF awards \#1637598 \& \#1645832, and support from Raytheon and Northrop Grumman. This research was also conducted in part at the Jet Propulsion Lab, California Insitute of Technology under a contract with the National Aeronautics and Space Administration.

\clearpage
\begin{small}
\bibliographystyle{plain}
\bibliography{reference}

\begin{thebibliography}{10}

\bibitem{abbeel2004apprenticeship}
Pieter Abbeel and Andrew~Y Ng.
\newblock Apprenticeship learning via inverse reinforcement learning.
\newblock In {\em International Conference on Machine Learning}, 2004.

\bibitem{Achterberg2009}
Tobias Achterberg.
\newblock {SCIP : solving constraint integer programs}.
\newblock {\em Mathematical Programming Computation}, 2009.

\bibitem{balcan2005co}
Maria-Florina Balcan, Avrim Blum, and Ke~Yang.
\newblock Co-training and expansion: Towards bridging theory and practice.
\newblock In {\em Neural information processing systems}, 2005.

\bibitem{balunovic2018learning}
Mislav Balunovic, Pavol Bielik, and Martin Vechev.
\newblock Learning to solve smt formulas.
\newblock In {\em Neural Information Processing Systems}, 2018.

\bibitem{blum2017efficient}
Avrim Blum and Yishay Mansour.
\newblock Efficient co-training of linear separators under weak dependence.
\newblock In {\em Conference on Learning Theory}, 2017.

\bibitem{blum1998combining}
Avrim Blum and Tom Mitchell.
\newblock Combining labeled and unlabeled data with co-training.
\newblock In {\em Conference on Learning Theory}, 1998.

\bibitem{boros1991max}
Endre Boros and Peter~L Hammer.
\newblock The max-cut problem and quadratic 0--1 optimization; polyhedral
  aspects, relaxations and bounds.
\newblock {\em Annals of Operations Research}, 1991.

\bibitem{brockman2016openai}
Greg Brockman, Vicki Cheung, Ludwig Pettersson, Jonas Schneider, John Schulman,
  Jie Tang, and Wojciech Zaremba.
\newblock Openai gym.
\newblock {\em arXiv}, 2016.

\bibitem{Burges1998}
Chris Burges, Erin Renshaw, and Matt Deeds.
\newblock Learning to rank using gradient descent.
\newblock In {\em International conference on Machine learning}, 1998.

\bibitem{chang2015learning}
Kai-Wei Chang, Akshay Krishnamurthy, Alekh Agarwal, Hal Daume, and John
  Langford.
\newblock Learning to search better than your teacher.
\newblock In {\em International Conference on Machine Learning}, 2015.

\bibitem{chen2011co}
Minmin Chen, Kilian~Q Weinberger, and John Blitzer.
\newblock Co-training for domain adaptation.
\newblock In {\em Neural information processing systems}, 2011.

\bibitem{cheng2018fast}
Ching-An Cheng, Xinyan Yan, Nolan Wagener, and Byron Boots.
\newblock Fast policy learning through imitation and reinforcement.
\newblock In {\em Conference on Uncertainty in Artificial Intelligence}, 2018.

\bibitem{cover2012elements}
Thomas~M Cover and Joy~A Thomas.
\newblock {\em Elements of information theory}.
\newblock John Wiley \& Sons, 2012.

\bibitem{Dai2016}
Hanjun Dai, Bo~Dai, and Le~Song.
\newblock {Discriminative Embeddings of Latent Variable Models for Structured
  Data}.
\newblock In {\em International Conference on Machine Learning}, pages 1--23,
  2016.

\bibitem{dai2017learning}
Hanjun Dai, Elias~B Khalil, Yuyu Zhang, Bistra Dilkina, and Le~Song.
\newblock Learning combinatorial optimization algorithms over graphs.
\newblock In {\em Neural Information Processing Systems}, 2017.

\bibitem{dasgupta2002pac}
Sanjoy Dasgupta, Michael~L Littman, and David~A McAllester.
\newblock Pac generalization bounds for co-training.
\newblock In {\em Neural information processing systems}, 2002.

\bibitem{daume2009search}
Hal Daum{\'e}, John Langford, and Daniel Marcu.
\newblock Search-based structured prediction.
\newblock {\em Machine learning}, 2009.

\bibitem{de2007linear}
Wenceslas~Fernandez de~la Vega and Claire Kenyon-Mathieu.
\newblock Linear programming relaxations of maxcut.
\newblock In {\em ACM-SIAM symposium on Discrete algorithms}, 2007.

\bibitem{duan2016benchmarking}
Yan Duan, Xi~Chen, Rein Houthooft, John Schulman, and Pieter Abbeel.
\newblock Benchmarking deep reinforcement learning for continuous control.
\newblock In {\em International Conference on Machine Learning}, 2016.

\bibitem{erdHos1960evolution}
Paul Erd{\H{o}}s and Alfr{\'e}d R{\'e}nyi.
\newblock On the evolution of random graphs.
\newblock {\em Publ. Math. Inst. Hung. Acad. Sci}, 1960.

\bibitem{greensmith2004variance}
Evan Greensmith, Peter~L Bartlett, and Jonathan Baxter.
\newblock Variance reduction techniques for gradient estimates in reinforcement
  learning.
\newblock {\em Journal of Machine Learning Research}, 2004.

\bibitem{gurobi}
LLC Gurobi~Optimization.
\newblock Gurobi optimizer reference manual, 2018.

\bibitem{he2014learning}
He~He, Hal Daume~III, and Jason~M Eisner.
\newblock Learning to search in branch and bound algorithms.
\newblock In {\em Neural information processing systems}, 2014.

\bibitem{henderson2018deep}
Peter Henderson, Riashat Islam, Philip Bachman, Joelle Pineau, Doina Precup,
  and David Meger.
\newblock Deep reinforcement learning that matters.
\newblock In {\em AAAI Conference on Artificial Intelligence}, 2018.

\bibitem{Hester2018}
Todd Hester, Olivier Pietquin, Marc Lanctot, Tom Schaul, Dan Horgan, John Quan,
  Andrew Sendonaris, Ian Osband, Gabriel Dulac-arnold, John Agapiou, and Joel~Z
  Leibo.
\newblock {Deep Q-Learning from Demonstrations}.
\newblock In {\em AAAI Conference on Artificial Intelligence}, 2018.

\bibitem{ho2016generative}
Jonathan Ho and Stefano Ermon.
\newblock Generative adversarial imitation learning.
\newblock In {\em Neural Information Processing Systems}, 2016.

\bibitem{johnson2017google}
Melvin Johnson, Mike Schuster, Quoc~V Le, Maxim Krikun, Yonghui Wu, Zhifeng
  Chen, Nikhil Thorat, Fernanda Vi{\'e}gas, Martin Wattenberg, Greg Corrado,
  et~al.
\newblock Google’s multilingual neural machine translation system: Enabling
  zero-shot translation.
\newblock {\em Transactions of the Association for Computational Linguistics},
  2017.

\bibitem{kakade2002approximately}
Sham Kakade and John Langford.
\newblock Approximately optimal approximate reinforcement learning.
\newblock In {\em International Conference on Machine Learning}, 2002.

\bibitem{kang2018policy}
Bingyi Kang, Zequn Jie, and Jiashi Feng.
\newblock Policy optimization with demonstrations.
\newblock In {\em International Conference on Machine Learning}, 2018.

\bibitem{khalil2016learning}
Elias~Boutros Khalil, Pierre Le~Bodic, Le~Song, George~L Nemhauser, and
  Bistra~N Dilkina.
\newblock Learning to branch in mixed integer programming.
\newblock In {\em AAAI Conference on Artificial Intelligence}, 2016.

\bibitem{kiritchenko2011email}
Svetlana Kiritchenko and Stan Matwin.
\newblock Email classification with co-training.
\newblock In {\em Conference of the Center for Advanced Studies on
  Collaborative Research}, 2011.

\bibitem{kober2013reinforcement}
Jens Kober, J~Andrew Bagnell, and Jan Peters.
\newblock Reinforcement learning in robotics: A survey.
\newblock {\em The International Journal of Robotics Research}, 2013.

\bibitem{kumar2011co}
Abhishek Kumar and Hal Daum{\'e}.
\newblock A co-training approach for multi-view spectral clustering.
\newblock In {\em International Conference on Machine Learning}, 2011.

\bibitem{lample2017playing}
Guillaume Lample and Devendra~Singh Chaplot.
\newblock Playing fps games with deep reinforcement learning.
\newblock In {\em AAAI Conference on Artificial Intelligence}, 2017.

\bibitem{land2010automatic}
Ailsa~H Land and Alison~G Doig.
\newblock An automatic method for solving discrete programming problems.
\newblock In {\em 50 Years of Integer Programming 1958-2008}, pages 105--132.
  Springer, 2010.

\bibitem{le2018hierarchical}
Hoang Le, Nan Jiang, Alekh Agarwal, Miroslav Dudik, Yisong Yue, and Hal
  Daum{\'e}.
\newblock Hierarchical imitation and reinforcement learning.
\newblock In {\em International Conference on Machine Learning}, 2018.

\bibitem{levine2016end}
Sergey Levine, Chelsea Finn, Trevor Darrell, and Pieter Abbeel.
\newblock End-to-end training of deep visuomotor policies.
\newblock {\em The Journal of Machine Learning Research}, 2016.

\bibitem{lillicrap2015continuous}
Timothy~P Lillicrap, Jonathan~J Hunt, Alexander Pritzel, Nicolas Heess, Tom
  Erez, Yuval Tassa, David Silver, and Daan Wierstra.
\newblock Continuous control with deep reinforcement learning.
\newblock In {\em International Conference on Learning Representations}, 2016.

\bibitem{linderoth2005simplicial}
Jeff Linderoth.
\newblock A simplicial branch-and-bound algorithm for solving quadratically
  constrained quadratic programs.
\newblock {\em Mathematical programming}, 2005.

\bibitem{liu2013multi}
Jialu Liu, Chi Wang, Jing Gao, and Jiawei Han.
\newblock Multi-view clustering via joint nonnegative matrix factorization.
\newblock In {\em SIAM International Conference on Data Mining}, 2013.

\bibitem{mirhoseini2017device}
Azalia Mirhoseini, Hieu Pham, Quoc~V Le, Benoit Steiner, Rasmus Larsen, Yuefeng
  Zhou, Naveen Kumar, Mohammad Norouzi, Samy Bengio, and Jeff Dean.
\newblock Device placement optimization with reinforcement learning.
\newblock In {\em International Conference on Machine Learning}, 2017.

\bibitem{mnih2013playing}
Volodymyr Mnih, Koray Kavukcuoglu, David Silver, Alex Graves, Ioannis
  Antonoglou, Daan Wierstra, and Martin Riedmiller.
\newblock Playing atari with deep reinforcement learning.
\newblock {\em arXiv}, 2013.

\bibitem{nair2018overcoming}
Ashvin Nair, Bob McGrew, Marcin Andrychowicz, Wojciech Zaremba, and Pieter
  Abbeel.
\newblock Overcoming exploration in reinforcement learning with demonstrations.
\newblock In {\em International Conference on Robotics and Automation}, 2018.

\bibitem{nigam2000analyzing}
Kamal Nigam and Rayid Ghani.
\newblock Analyzing the effectiveness and applicability of co-training.
\newblock In {\em ACM Conference on Information and knowledge Management},
  2000.

\bibitem{oncan2009comparative}
Temel {\"O}ncan, {\.I}~Kuban Alt{\i}nel, and Gilbert Laporte.
\newblock A comparative analysis of several asymmetric traveling salesman
  problem formulations.
\newblock {\em Computers \& Operations Research}, 2009.

\bibitem{ono2008efficient}
Masahiro Ono and Brian~C Williams.
\newblock An efficient motion planning algorithm for stochastic dynamic systems
  with constraints on probability of failure.
\newblock In {\em AAAI Conference on Artificial Intelligence}, 2008.

\bibitem{orman2007survey}
AJ~Orman and HP~Williams.
\newblock A survey of different integer programming formulations of the
  travelling salesman problem.
\newblock In {\em Optimisation, econometric and financial analysis}. Springer,
  2007.

\bibitem{puterman2014markov}
Martin~L Puterman.
\newblock {\em Markov decision processes: discrete stochastic dynamic
  programming}.
\newblock John Wiley \& Sons, 2014.

\bibitem{ross2010efficient}
St{\'e}phane Ross and Drew Bagnell.
\newblock Efficient reductions for imitation learning.
\newblock In {\em International Conference on Artificial Intelligence and
  Statistics}, 2010.

\bibitem{ross2014reinforcement}
Stephane Ross and J~Andrew Bagnell.
\newblock Reinforcement and imitation learning via interactive no-regret
  learning.
\newblock {\em arXiv}, 2014.

\bibitem{ross2011reduction}
St{\'e}phane Ross, Geoffrey Gordon, and Drew Bagnell.
\newblock A reduction of imitation learning and structured prediction to
  no-regret online learning.
\newblock In {\em International Conference on Artificial Intelligence and
  Statistics}, 2011.

\bibitem{schulman2015trust}
John Schulman, Sergey Levine, Philipp Moritz, Michael Jordan, and Pieter
  Abbeel.
\newblock Trust region policy optimization.
\newblock In {\em International Conference on Machine Learning}, 2015.

\bibitem{schulman2017proximal}
John Schulman, Filip Wolski, Prafulla Dhariwal, Alec Radford, and Oleg Klimov.
\newblock Proximal policy optimization algorithms.
\newblock {\em arXiv}, 2017.

\bibitem{silver2016mastering}
David Silver, Aja Huang, Chris~J Maddison, Arthur Guez, Laurent Sifre, George
  Van Den~Driessche, Julian Schrittwieser, Ioannis Antonoglou, Veda
  Panneershelvam, Marc Lanctot, et~al.
\newblock Mastering the game of go with deep neural networks and tree search.
\newblock {\em Nature}, 2016.

\bibitem{sindhwani2005co}
Vikas Sindhwani, Partha Niyogi, and Mikhail Belkin.
\newblock A co-regularization approach to semi-supervised learning with
  multiple views.
\newblock In {\em ICML workshop on learning with multiple views}, 2005.

\bibitem{song2018learning}
Jialin Song, Ravi Lanka, Albert Zhao, Aadyot Bhatnagar, Yisong Yue, and
  Masahiro Ono.
\newblock Learning to search via retrospective imitation.
\newblock {\em arXiv}, 2018.

\bibitem{stadie2017third}
Bradly~C Stadie, Pieter Abbeel, and Ilya Sutskever.
\newblock Third-person imitation learning.
\newblock {\em arXiv}, 2017.

\bibitem{sun2017deeply}
Wen Sun, Arun Venkatraman, Geoffrey~J Gordon, Byron Boots, and J~Andrew
  Bagnell.
\newblock Deeply aggrevated: Differentiable imitation learning for sequential
  prediction.
\newblock In {\em International Conference on Machine Learning}, 2017.

\bibitem{sutton2000policy}
Richard~S Sutton, David~A McAllester, Satinder~P Singh, and Yishay Mansour.
\newblock Policy gradient methods for reinforcement learning with function
  approximation.
\newblock In {\em Neural information processing systems}, 2000.

\bibitem{syed2008apprenticeship}
Umar Syed, Michael Bowling, and Robert~E Schapire.
\newblock Apprenticeship learning using linear programming.
\newblock In {\em International Conference on Machine Learning}, 2008.

\bibitem{syed2008game}
Umar Syed and Robert~E Schapire.
\newblock A game-theoretic approach to apprenticeship learning.
\newblock In {\em Neural information processing systems}, 2008.

\bibitem{todorov2012mujoco}
Emanuel Todorov, Tom Erez, and Yuval Tassa.
\newblock Mujoco: A physics engine for model-based control.
\newblock In {\em International Conference on Intelligent Robots and Systems},
  2012.

\bibitem{van2016deep}
Hado Van~Hasselt, Arthur Guez, and David Silver.
\newblock Deep reinforcement learning with double q-learning.
\newblock In {\em AAAI Conference on Artificial Intelligence}, 2016.

\bibitem{wan2009co}
Xiaojun Wan.
\newblock Co-training for cross-lingual sentiment classification.
\newblock In {\em Joint conference of ACL and IJCNLP}. Association for
  Computational Linguistics, 2009.

\bibitem{wang2010new}
Wei Wang and Zhi-Hua Zhou.
\newblock A new analysis of co-training.
\newblock In {\em International Conference on Machine Learning}, 2010.

\bibitem{wang2013co}
Wei Wang and Zhi-Hua Zhou.
\newblock Co-training with insufficient views.
\newblock In {\em Asian conference on machine learning}, 2013.

\bibitem{wang2016dueling}
Ziyu Wang, Tom Schaul, Matteo Hessel, Hado Hasselt, Marc Lanctot, and Nando
  Freitas.
\newblock Dueling network architectures for deep reinforcement learning.
\newblock In {\em International Conference on Machine Learning}, 2016.

\bibitem{ziebart2008maximum}
Brian Ziebart, Andrew Maas, J~Andrew Bagnell, and Anind Dey.
\newblock Maximum entropy inverse reinforcement learning.
\newblock In {\em AAAI Conference on Artificial Intelligence}, 2008.

\end{thebibliography}
\end{small}
\clearpage
\section{APPENDIX}
\label{sec:appendix}

\subsection{PROOFS}
\label{sec:sec:proof}
\textbf{Proof for Proposition \ref{proposition:divergence}:}
\begin{proof}
We show that $\max\nolimits_{s} D_{JS}(\pi^B(s)\|\pi^A(s))$ is well-defined for an MDP $\mdp$ with two representations $\mdp^A$ and $\mdp^B$. From Theorem \ref{theorem:occupancy}, we know the distribution $\pi(s)$ can be written with respect to its occupancy measure $\rho_{\pi}$. It is sufficient to show that we can map occupancy measures of $\pi^A$ and $\pi^B$ to a common MDP. By the definition of an occupancy measure, 
\begin{align*}
\rho_{\pi}(s, a) &= \mathbb{P}(\pi(s)=a)\sum\nolimits_{i=0}^{\infty}\gamma^i \mathbb{P}(s_i=s|\pi) \\
                         &= \expect{\tau=(s_0, a_0, \cdots, s_n)\sim \pi}{\sum\nolimits_{i=0}^n \gamma^i \textbf{1}((s_i, a_i) = (s, a))} 
\end{align*}
that is to say, the occupancy measure is the expected discounted count of a state-action pair to appear in all possible trajectories. Since we have trajectory mappings between $\mdp^A$ and $\mdp^B$, we can convert an occupancy measure in $\mdp^A$ to one in $\mdp^B$ by mapping each trajectory and perform the count in the new MDP representation. Formally, the occupancy measure $\rho^B_{\pi^B}$ of $\pi^B$ in $\mdp^B$ can be mapped to an occupancy measure in $\mdp^A$ by
\begin{align*}
    &\rho^A_{\pi^B}(s, a) \\
    &= \expect{\substack{\tau^B\sim \pi^B, \\ f_{B\rightarrow A}(\tau^B)=(s_0, a_0, \cdots, s_n)}}{\sum\nolimits_{i=0}^n \gamma^i \textbf{1}((s_i, a_i) = (s, a))} 
\end{align*}
Following from this, we can compute $D_{JS}(\pi^B(s)\|\pi^A(s))$ using any $s$ in $\mdp^A$. And the maximum is defined. In the definition, there is a choice whether to map $\pi^A$'s occupancy measure to $\mdp^B$ or $\pi^B$'s to $\mdp^A$. Though both approaches lead to a valid definition, we use the definition that for $D_{JS}(\cdot\|\cdot)$, we always map the representation in the first argument to that of the second argument. It is preferable to the other one because in Theorem \ref{theorem:general}, we want to optimize 
\begin{equation*}
    J(\pi'^A) \ge J_{\pi^A}(\pi'^A) - \frac{2\gamma^A(4\beta^B_{\dist_2} \epsilon^B_{\dist_2} + \alpha^A_{\dist} \epsilon^{A}_{\dist})}{(1-\gamma^A)^2} + \delta_2
\end{equation*}
by optimizing 
\begin{equation*}
\beta^B_{\dist_2}=\expect{\mdp\sim\dist_2}{\max\nolimits_{s} D_{JS}(\pi^B(s)\|\pi^A(s))}    
\end{equation*}
usually via computing the gradient of $\beta^B_{\dist_2}$ w.r.t. $\pi^A$. If we use $f_{A\rightarrow B}$ to map from $\mdp^A$ to $\mdp^B$, the gradient will involve a complex composition of $f_{A\rightarrow B}$ and $\pi^A$, which is undesirable. 
\end{proof}

To prove Theorem \ref{theorem:general}, we need to use a policy improvement result for a single MDP (a modified version of Theorem 1 in \citep{kang2018policy}).
\begin{theorem}
Assume for an MDP $\mdp$, an expert policy $\pi_E$ have a higer advantage of over a policy $\pi$ with a margin, i.e.,  $\eta(\pi_E, \mathcal{M}) - \eta(\pi, \mathcal{M}) \ge \delta$
Define
\begin{align*}
&\alpha = \max\nolimits_{s} D_{KL}(\pi'(s)\|\pi(s)) \\
&\beta = \max\nolimits_{s} D_{JS}(\pi'(s)\|\pi_E(s)) \\
&\epsilon_{\pi_E} = \max\nolimits_{s, a} |A_{\pi_E}(s, a)| \\
&\epsilon_{\pi} = \max\nolimits_{s, a} |A_{\pi}(s, a)|
\end{align*}
then $\eta(\pi', \mdp) \ge \eta_{\pi}(\pi', \mdp) - \frac{2\gamma(4\beta\epsilon_{\pi_E}+\alpha\epsilon_{\pi})}{(1-\gamma)^2} + \delta$
\end{theorem}
\begin{proof}
The only difference from the original theorem is that the original assumes $\expect{a_E\sim\pi_E(s), a\sim\pi(s)}{A_{\pi}(s, a_E) - A_{\pi}(s, a)} \ge \delta' > 0$ for every state $s$. It is a stronger assumption which is not needed in their analysis. Notice that the advantage of a policy over itself is zero, i.e., $\expect{a\sim\pi(s)}{A_{\pi}(s, a)} = 0$ for every $s$, so the margin assumption simplifies to $\expect{a_E\sim\pi_E(s)}{A_{\pi}(s, a_E)} \ge \delta'$.

By the policy advantage formula,  
\begin{align*}
    \eta(\pi_E, \mathcal{M}) - \eta(\pi, \mathcal{M}) &= \mathbb{E}_{\tau\sim\pi_E}[\sum\nolimits_{i=0}^{\infty}\gamma^i A_{\pi}(s_i, a_i)] \\
    &= \mathbb{E}_{s_i\sim\rho_{\pi_E}}\expect{a_i\sim\pi_E(s_i)}{\sum\nolimits_{i=0}^{\infty}\gamma^i A_{\pi}(s_i, a_i)} \\
    &\ge \mathbb{E}_{s_i\sim\rho_{\pi_E}} [\delta' \sum\nolimits_{i=0}^{\infty} \gamma^i] \\
    &= \frac{\delta'}{1-\gamma}
\end{align*}
So an assumption on per-state advantage translates to a overall advantage. Thus we can make this weaker assumption which is also more intuitive and the original statement still holds with a different $\delta$ term.
\end{proof}

\textbf{Proof of Theorem \ref{theorem:general}:}
\begin{proof}
Theorem \ref{theorem:general} is a distributional extension to the theorem above. For $\mdp\sim\dist_2$, let $\delta_{\mdp} = \eta(\pi^B, \mdp^B) - \eta(\pi^A, \mdp^A)$.
\begin{align*}
    &J(\pi'^A) \\
              &= \expect{\mdp\sim\dist}{\eta(\pi'^A, \mdp^A)} \\
              &= \expect{\mdp\sim\dist_1}{\eta(\pi'^A, \mdp^A)} + \expect{\mdp\sim\dist_2}{\eta(\pi'^A, \mdp^A)} \\
              &\ge \expect{\mdp\sim\dist_1}{\eta(\pi'^A, \mdp^A)} + \\
              &\expect{\mdp\sim\dist_2}{\eta_{\pi^A}(\pi'^A, \mdp^A) - \frac{2\gamma^A(4\beta\epsilon_{\pi^B}+\alpha\epsilon_{\pi^A})}{(1-\gamma^A)^2} + 
              \delta_{\mdp}} \\
              &\ge \expect{\mdp\sim\dist_1}{\eta_{\pi^A(}(\pi'^A, \mdp^A) - \frac{2\gamma^A\alpha\epsilon_{\pi^A}}{(1-\gamma^A)^2}} + \\
              &\expect{\mdp\sim\dist_2}{\eta_{\pi^A}(\pi'^A, \mdp^A) - \frac{2\gamma^A(4\beta\epsilon_{\pi^B}+\alpha\epsilon_{\pi^A})}{(1-\gamma^A)^2} + 
              \delta_{\mdp}} \\
              &= \expect{\mdp\sim\dist}{\eta_{\pi^A}(\pi'^A, \mdp^A)} - \expect{\mdp\sim\dist}{\frac{2\gamma^A\alpha\epsilon_{\pi^A}}{(1-\gamma^A)^2}} -\\
              &\expect{\mdp\sim\dist_2}{\frac{2\gamma^A\cdot4\beta\epsilon_{\pi^B}}{(1-\gamma^A)^2}} + \expect{\mdp\sim\dist_2}{\delta_{\mdp}} \\
              &\ge J_{\pi^A}(\pi'^A) - \frac{2\gamma^A(4\beta^B_{\dist_2} \epsilon^B_{\dist_2} + \alpha^A_{\dist} \epsilon^{A}_{\dist})}{(1-\gamma^A)^2} + \delta_2
\end{align*}
The derivation for $J(\pi'^B)$ is the same.
\end{proof}

Finally, we provide the proof for Theorem \ref{theorem:special}. We first quantify the performance gap between a policy $\pi$ and an optimal policy $\pi^*$. For a policy that is able to achieve $\epsilon$ $0-1$ loss, $\ell(s, \pi) = \textbf{1}(\pi(s)\ne \pi^*(s))$, measured against $\pi^*$'s action choices under its own state distributions, then we can bound the performance gap. Let $Q_t^{\pi'}(s, \pi)$ denote the $t$-step cost of executing $\pi$ in initial state $s$ and then following policy $\pi'$

\begin{theorem} 
\label{theorem:gap}
(Theorem 2.2 from \citep{ross2011reduction}, adpated to our notations) Let $\pi$ be such taht $\expect{s\sim \rho_{\pi}}{\ell(s, \pi)}=\epsilon$, and $Q^{\pi^*}_{T-t+1}(s, \pi^*) - Q^{\pi^*}_{T-t+1}(s, a) \le u$ for all action $a, t\in\{1, 2, \cdots, T\}$, then $\eta(\pi, \mdp) \ge \eta(\pi^*, \mdp) - uT\epsilon$.
\end{theorem}
Thus the important quantity to measure is $\epsilon$, and by measuring the disagreements between two policies in two views, we can upper bound $\epsilon$. The result is originally stated in the context of classification, and the above theorem justifies the learning reduction approach of reducing policy learning to classification. 

\begin{theorem} 
\label{theorem:pac}
(Corollary 5 in \citep{dasgupta2002pac} applied to full classifiers) Using the definitions in Theorem \ref{theorem:special}, with probability $1-\sigma$ over the choice of a sample set $N$, for all pairs of classifiers $h_1, h_2$ such that for all $i$ we have $\zeta_i(h_1, h_2, \sigma) > 0$ and $b_i(h_1, h_2, \sigma)\le 1$.
\begin{equation*}
    \epsilon \le \max\nolimits_{j\in\{1, \cdots, k\}}b_j(h_1, h_2, \sigma)
\end{equation*}
\end{theorem}
\begin{proof}
The only change from the original proof is that instead of a partial classifier which can output $\bot$, we consider a full classifier. Then we could eliminate the estimates for $\mathbb{P}(h_1\ne\bot)$ and the error introduced by converting a partial classifier to a full classifier via random labelling when the output is $\bot$.
\end{proof}

\textbf{Proof of Theorem \ref{theorem:special}:}
\begin{proof}
For the bound for $\pi^A$, we are measuring $\epsilon_A$ on its sampled paths. Then directly apply Theorem \ref{theorem:pac} gives an upper bound on $\epsilon_A$. Apply Theorem \ref{theorem:gap} gives the result of Theorem \ref{theorem:special}.
\end{proof}

\subsection{PICTORIAL REPRESENTATION OF THE TWO-VIEWS IN RISK-AWARE PATH PLANNING:}
\label{sec:pic_rep_psulu}

We present a pictorial representation of the two different views used in the experiments in Fig~\ref{fig:exp:puslu_rep}. In the MILP view, the constraint space is represented using additional auxiliary binary variables to choose the active side of the polytope, whereas in the QCQP view, the constraint space can be encoded in a quadratic constraint.

\begin{figure}[t]
    \includegraphics[trim={0pt 0pt 0pt 0pt}, width=0.45\textwidth]{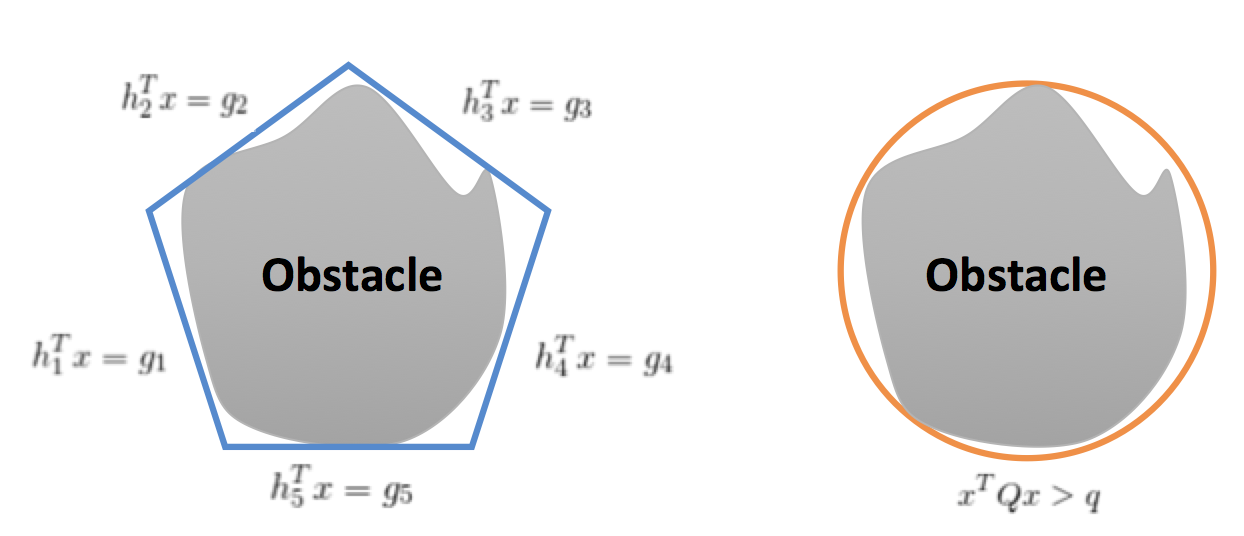}
    \caption{Two views for Risk-Aware Path Planning. On the left, the obstacle is enclosed by a polytope (MILP view) and on the right the obstacle is enclosed by an ellipse (QCQP view).}
    \label{fig:exp:puslu_rep}
\end{figure}

\subsection{RISK-AWARE PLANNING DATASET GENERATION:}
\label{sec:data_gen_psulu}
We generate 150 obstacle maps. Each map contains 10 rectangle obstacles, with the center of each obstacle chosen from a uniform random distribution over the space $0\leq y \leq 1$, $0 \leq x \leq 1$. The side length of each obstacle was chosen from a uniform distribution in range [0.01, 0.02] and the orientation was chosen from a uniform distribution between $\ang{0}$ and $\ang{360}$. In order to avoid trivial infeasible maps, any obstacles centered close to the destination are removed. For MILP view, we directly use the randomly generated rectangles for defining the constraint space. However, for the QCQP view, we enclose the rectangle obstacles with a circle for defining the quadratic constraint.

\subsection{DISCRETE/CONTINUOUS CONTROL RESULTS IN TABULAR FORM}
\label{sec:control:tabular}
\begin{tabular}{|c|c|c|c|}
    \hline 
    & Acrobot & Swimmer & Hopper \\
    \hline
    A (CoPiEr) & $-86.44 \pm 10.80$ & $106.35\pm 23.11$ & $217.83\pm 30.03$ \\
    \hline 
    A (PG) &  $-169.57\pm 10.48$ & $109.09 \pm 21.58$ & $278.66\pm 32.87$ \\
    \hline 
    A (All) & $-252.42 \pm 8.73$ & $100.36\pm 22.37$ & $49.39\pm 10.35$ \\
    \hline 
    B (CoPiEr) & $-88.48 \pm 15.13$ & $104.16\pm 19.32$ & $168.88\pm 18.21$ \\
    \hline 
    B (PG) & $-257.16\pm 10.93$ & $103.48\pm 21.89$ & $89.34\pm 4.89$ \\
    \hline 
    B (All) & $-251.74\pm 9.65$ & $96.74\pm 19.57$ & $22.59\pm 5.55$ \\
    \hline 
    A + B & $-86.42\pm 3.48$ & $108.71\pm 5.03$ & $346.53\pm 5.91$\\
    \hline
\end{tabular}

\end{document}